\documentclass{article}

\usepackage{natbib}
\usepackage[preprint]{neurips_2022}

\usepackage{microtype}
\usepackage{graphicx}
\usepackage{subfigure}
\usepackage{wrapfig}
\usepackage{booktabs} 
\usepackage[normalem]{ulem}
\usepackage{amsmath}
\usepackage{amssymb}
\usepackage{mathtools}
\usepackage{amsthm}
\usepackage{mathrsfs}
\usepackage{gensymb}

\usepackage{amsmath,amsfonts,bm}

\def\eqref#1{equation~\ref{#1}}

\def\1{\bm{1}}

\def\ve{{\bm{e}}}

\def\vs{{\bm{s}}}

\DeclareMathAlphabet{\mathsfit}{\encodingdefault}{\sfdefault}{m}{sl}
\SetMathAlphabet{\mathsfit}{bold}{\encodingdefault}{\sfdefault}{bx}{n}

\usepackage[utf8]{inputenc} 
\usepackage[T1]{fontenc}    
\usepackage{hyperref}       
\usepackage{url}            
\usepackage{booktabs}       
\usepackage{amsfonts}       
\usepackage{nicefrac}       
\usepackage{microtype}      
\usepackage{xcolor}         

\theoremstyle{plain}
\newtheorem{theorem}{Theorem}[section]

\newtheorem{definition}[theorem]{Definition}

\theoremstyle{remark}

\usepackage{multirow}

\title{Two-Dimensional Weisfeiler-Lehman Graph Neural Networks for Link Prediction}

\author{%
  Yang Hu$^{1}$~~~~Xiyuan Wang$^{1}$~~~~Zhouchen Lin$^{2,1}$~~~~Pan Li$^{3}$~~~~Muhan Zhang$^{1,4,}$\thanks{Corresponding author: Muhan Zhang (\texttt{muhan@pku.edu.cn}).}\\
  ${}^1$Institute for Artificial Intelligence, Peking University\\ 
  ${}^2$Key Lab. of Machine Perception, School of Artificial Intelligence, Peking University\\ 
  ${}^3$Department of Computer Science, Purdue University\\ 
  ${}^4$Beijing Institute for General Artificial Intelligence\\
  \texttt{huyang0625@pku.edu.cn, wangxiyuan@pku.edu.cn}\\ \texttt{zlin@pku.edu.cn, panli@purdue.edu, muhan@pku.edu.cn}
}

\begin{document}

\maketitle


\begin{abstract}
Link prediction is one important application of graph neural networks (GNNs). Most existing GNNs for link prediction are based on one-dimensional Weisfeiler-Lehman ($1$-WL) test. $1$-WL-GNNs first compute node representations by iteratively passing neighboring node features to the center, and then obtain link representations by aggregating the pairwise node representations. As pointed out by previous works, this two-step procedure results in low discriminating power, as $1$-WL-GNNs by nature learn node-level representations instead of link-level. In this paper, we study a completely different approach which can directly obtain node pair (link) representations based on \textit{two-dimensional Weisfeiler-Lehman ($2$-WL) tests}. $2$-WL tests directly use links (2-tuples) as message passing units instead of nodes, and thus can directly obtain link representations. We theoretically analyze the expressive power of $2$-WL tests to discriminate non-isomorphic links, and prove their superior link discriminating power than $1$-WL. Based on different $2$-WL variants, we propose a series of novel $2$-WL-GNN models for link prediction. Experiments on a wide range of real-world datasets demonstrate their competitive performance to state-of-the-art baselines and superiority over plain $1$-WL-GNNs.





\end{abstract}

\section{Introduction}

Link prediction is a key problem of graph-structured data~\citep{al2006link,liben2007link,menon2011link,trouillon2016complex}. It refers to utilizing node characteristics and graph topology to measure how likely a link exists between a pair of nodes. Due to the importance of predicting pairwise relations, it has wide applications in various domains, such as recommendation in social networks~\citep{adamic2003friends}, knowledge graph completion~\citep{nickel2015review}, and metabolic network reconstruction~\citep{oyetunde2017boostgapfill}.

One class of traditional link prediction methods are heuristic methods, which use manually designed graph structural features of a target node pair such as number of common neighbors (CN)~\citep{liben2007link}, preferential attachment (PA)~\citep{barabasi1999emergence}, and resource allocation (RA)~\citep{zhou2009predicting} to estimate the likelihood of link existence. Another class of methods, embedding methods, including Matrix Factorization (MF)~\citep{menon2011link} and node2vec~\citep{grover2016node2vec}, learn node embeddings from the graph structure in a transductive manner, which cannot generalize to unseen nodes or new graphs. Recently, with the popularity of GNNs, their application to link prediction brings a number of cutting-edge models~\citep{kipf2016variational,zhang2018link,zhang2021labeling,zhu2021neural}.

Most existing GNN models for link prediction are based on one-dimensional Weisfeiler-Lehman ($1$-WL) test~\citep{shervashidze2011weisfeiler}. $1$-WL test is a popular heuristic for detecting non-isomorphic graphs. In each update, it obtains all nodes' new colors by hashing their own colors and multisets of their neighbors' colors. Vanilla GNNs simulate $1$-WL test by iteratively aggregating neighboring node features to the center node to update node representations, which we call $1$-WL-GNNs. With the node representations, $1$-WL-GNNs compute link prediction scores by aggregating pairwise node representations. Graph AutoEncoder (GAE, and its variant VGAE)~\citep{kipf2016variational} is such a model. However, $1$-WL-GNNs can only discriminate links on the ``node'' level. If two nodes $v_2$ and $v_3$ have symmetrical positions in graph $G$, then for another node $v_1$ in $G$, GAE cannot distinguish links $(v_1, v_2)$ and $(v_1, v_3)$, though they may not be symmetrical links in graph $G$. See Figure~\ref{fig:1} left. Although positional node embeddings or random features can alleviate this problem, they fail to guarantee symmetrical links (such as $(v_1, v_2)$ and $(v_4, v_3)$) to have the same representation.

\begin{wrapfigure}[11]{l}{7.3cm}
\vspace{-11pt}
\centering
\includegraphics[height=2.3cm]{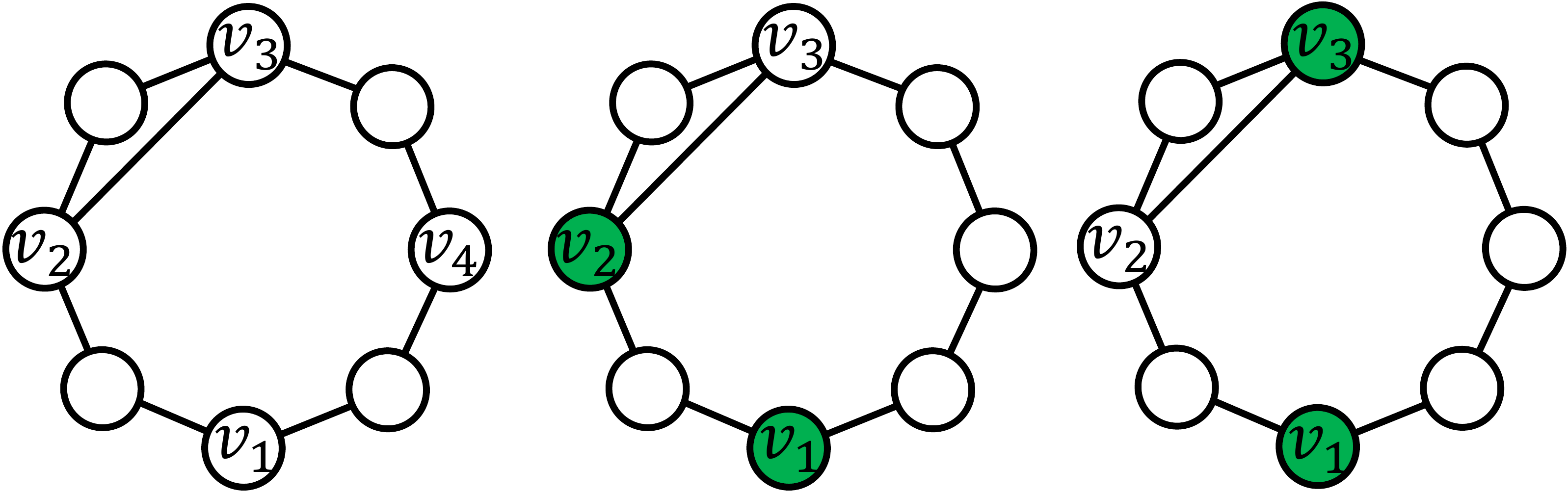}
\vspace{-18pt}
\caption{$1$-WL-GNNs cannot distinguish links $(v_1, v_2)$ and $(v_1, v_3)$ in the left graph. With labeling trick, $1$-WL-GNNs can distinguish them in their respective labeled graphs (middle and right).}
\label{fig:1}
\end{wrapfigure}
Another classic model SEAL~\citep{zhang2018link} applies labeling trick~\citep{zhang2021labeling} to augment input node features with target-link-specific node labels, and applies a $1$-WL-GNN to each link's enclosing subgraph to learn the link representation. Labeling trick brings asymmetry between the target node pair and all other nodes during the message massing. Although labeling trick raises the link discriminating power from ``node'' to ``link'' level, it requires extracting a subgraph and repeatedly applying GNN for every link to predict. We include more discussion in Section~\ref{sec:labelingtrick}. 

In this paper, we propose a completely different approach for link prediction. We construct GNNs based on two-dimensional Weisfeiler-Lehman ($2$-WL) tests, which we call $2$-WL-GNNs. In $2$-WL-GNNs, node pairs are used as the elemental message passing units, so that link representations are directly obtained instead of aggregating pairwise node representations as in $1$-WL-GNNs. Figure~\ref{fig:2fwll} gives an illustration for a particular $2$-WL. We first theoretically study the link discriminating power of different $2$-WL test variants, including the plain $2$-WL test, $2$-FWL (Folklore WL), local $2$-WL, and local $2$-FWL. We show that $2$-WL, $2$-FWL and local $2$-FWL are strictly more expressive than $1$-WL for link prediction, while local $2$-WL has equivalent link discriminating power to $1$-WL. Based on these $2$-WL tests, we propose a series of $2$-WL-GNN models. Despite all using node pairs to propagate messages, these models have different aggregation schemes, link discriminating power, time/space complexity, as well as drastically different implementations, which we discuss in Section~\ref{sec:implementation}. Extensive experiments on multiple benchmark datasets verify $2$-WL-GNNs' power for link prediction. We show that $2$-WL-GNNs achieve highly competitive link prediction performance to state-of-the-art models including SEAL and NBFNet~\citep{zhu2021neural}, while using significantly less time. 




\section{Limitations of using $1$-WL for link prediction}
In this section, we show the fundamental limitations of using $1$-WL-GNNs for link prediction. 
We denote a set by $\{a,b,c,...\}$, an ordered set (tuple) by $(a,b,c,...)$ and a multiset by $\{\!\!\{a,b,c,...\}\!\!\}$, where a multiset is allowed to have repeated elements. We use $[n]$ to denote the set $\{1,2,...,n\}$. Let $G = (V, E, l)$ be a labeled graph, where $V=[n]$ is the node set and $E \subseteq [n]\times[n]$ is the edge set, and $l: V\rightarrow \Sigma$ gives each node an initial label from $\Sigma$. Weisfeiler-Lehman tests are a series of algorithms to determine non-isomorphic graphs. In the base case of $1$-WL, at the beginning, every node $v$ has its representation (color) $c^{(0)}(v) = l(v)$. In iteration $t$, the representation of node $v$ is updated by $c^{(t)}(v) = f(c^{(t-1)}(v), \{\!\!\{c^{(t-1)}(u) | u\in N(v) \}\!\!\})$, where $f$ is an injective function and $N(v)$ denotes $v$'s neighbors.
$1$-WL can detect two non-isomorphic graphs if they have different multisets of node representations in some iteration. $1$-WL-GNN implements $f$ with neural networks.



\subsection{$1$-WL cannot learn link-level representations}
GAE is a representative $1$-WL-GNN model for link prediction. GAE first uses a vanilla $1$-WL-GNN to compute node representations, and then aggregates two node representations to obtain their link representation. \citet{zhang2021labeling} have studied that the direct aggregation has only node-level discriminating power. This is illustrated by Figure \ref{fig:1} left: $v_2$ and $v_3$ are symmetric nodes in the graph thus having the same representation by $1$-WL-GNN, but links $(v_1, v_2)$ and $(v_1, v_3)$ are not symmetric. However, by aggregating pairwise node representations as link representations, $1$-WL-GNNs are unable to discriminate links $(v_1, v_2)$ and $(v_1, v_3)$, though $(v_1, v_2)$ has a shorter path between them than $(v_1, v_3)$. $1$-WL-GNNs cannot even \textbf{learn common neighbors} between two nodes, which is a crucial link prediction heuristic for many networks. By computing node representations independently of each other, they completely lose the conditional information and dependence between the two target nodes. This reveals one fundamental limitation of using $1$-WL for link prediction: $1$-WL discriminates links only at the \textbf{node level}---it fails to learn link representations \textbf{as a whole}, but regards two links as indistinguishable if the corresponding nodes from the two links are indistinguishable.

\subsection{Labeling trick enhances $1$-WL-GNNs' link discriminating power}\label{sec:labelingtrick}
There are many link prediction models that apply labeling trick inherently, including SEAL~\citep{zhang2018link}, Distance Encoding~\citep{li2020distance}, ID-GNN~\citep{you2021identity}, and some models for matrix completion~\citep{Zhang2020Inductive} and knowledge graph completion~\citep{teru2020inductive}. These methods all label nodes according to their relation to the target link and then apply $1$-WL-GNN to the labeled graph. The target node representations obtained in the labeled graph are then aggregated into the link representation. The inherent mechanism, labeling trick~\citep{zhang2021labeling}, is proved to significantly enhance the link discriminating power of $1$-WL-GNNs, and ultimately promote a node-most-expressive GNN to be link-most-expressive thus theoretically closing the gap between GNN's node representation learning nature and link prediction's link representation requirement.
Figure~\ref{fig:1} middle and right illustrate this effect. When predicting $(v_1, v_2)$, $v_1$ and $v_2$ are labeled differently from other nodes; when predicting $(v_1, v_3)$, $v_1$ and $v_3$ are labeled differently from other nodes. Thus, links $(v_1, v_2)$ and $(v_1, v_3)$ can be differentiated by applying $1$-WL to their respective labeled graphs, as the labeled $v_1$ will pass its message to $v_2$ and $v_3$ with different number of steps. 



Although labeling trick brings fundamental improvement to GNN's link discriminating power, it also introduces some challenge. Labeling trick requires repeatedly applying GNN to a labeled subgraph for \textbf{every} link to predict. This is in contrast to GAE which can apply GNN to the entire graph only once and simultaneously learn representations for all target links. 
In other words, labeling trick methods lose the ability to obtain all link predictions in a single GNN inference step, and therefore often have low efficiency. In this paper, we aim to develop novel GNN models with both full-batch link prediction ability and higher expressive power than $1$-WL.



\section{Two-dimensional Weisfeiler-Lehman tests}


In $k$-dimensional WL test ($k$-WL), the unit of representation update becomes $k$-tuples of nodes. When $k=2$, every node pair updates its representation from its ``neighboring'' node pairs. Thus, node pair representations can be directly obtained for link prediction. $k$-WL-GNNs inherit $k$-WL by making the update functions neural networks. Since neural networks are universal approximators, $k$-WL-GNNs have the same maximum expressive power as $k$-WL~\citep{xu2018powerful,morris2019weisfeiler} in graph-level tasks. 
\begin{figure}[t]
\centering
\includegraphics[height=2.5cm]{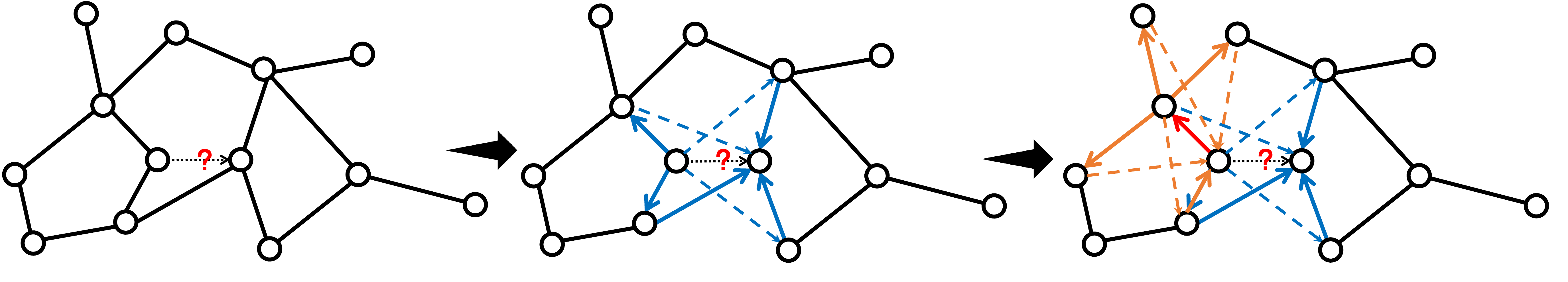}
\caption{This figure illustrates how local $2$-FWL, a particular $2$-WL test, works for link prediction. It takes links as message passing units. Given the centering directed link to predict, local $2$-FWL aggregates information from its neighboring links in each iteration. The neighboring links in the first iteration are shown as blue lines. Note that even unconnected links can be involved. In the second iteration, information from farther links are aggregated. We take the red link as an example and show the $2$-hop link neighborhood connected by this link as orange lines.}
\label{fig:2fwll}
\end{figure}
There are two variants of $k$-WL algorithms: the plain $k$-dimensional WL ($k$-WL) and the $k$-dimensional Folklore WL ($k$-FWL)~\citep{cai1992optimal,grohe2017descriptive}. In the following, we will use $k$-WL to specifically denote the plain version and use $k$-FWL to denote the Folklore version. Both $k$-WL and $k$-FWL update representations for $k$-tuples of nodes, where a $k$-tuple $\vs$ is defined by $\vs:=(s_1,s_2,...,s_k)$ with $s_1,...,s_k$ being nodes. 

$k$-WL defines the \textit{$j$th neighborhood} of $k$-tuple $\vs$ as
\begin{align}
N_j(\vs) = \big\{\!\!\big\{(s_1,...,s_{j-1},s',s_{j+1},...,s_k)\vert s'\in[n]\big\}\!\!\big\}. 
\end{align}

That is, the $j$th neighbors of $\vs$ are obtained by replacing the $j$th element of $\vs$ by $s'\in[n]$. And the \textit{full neighborhood} of $\vs$ is defined by $N(\vs) = \big(N_1(\vs),N_2(\vs),...,N_k(\vs)\big)$. Therefore, $k$-WL has $k$ fine-grained neighborhoods $N_j(\vs),j\in[k]$, and each fine-grained neighborhood has $n$ $k$-tuples.

$k$-FWL has a different definition of neighborhood. $k$-FWL defines the $j$th neighborhood of $\vs$ as 
\begin{align}
N^F_j(\vs)=\big((j,s_2,...,s_k),(s_1,j,...,s_k),...,(s_1,...,s_{k-1},j)\big).
\end{align}
And the full neighborhood of $\vs$ is given by $N^F(\vs) =\{\!\!\{N^F_j(\vs)\vert j\in [n]\}\!\!\}$.
That is, $k$-FWL has $n$ fine-grained neighborhoods $N^F_j(\vs), j\in [n]$, and the $j$th fine-grained neighborhood $N^F_j(\vs)$ is obtained by iteratively replacing each element of $\vs$ by node $j$. Essentially, $k$-WL and $k$-FWL have the same $nk$ neighbor tuples but differ in how these $nk$ tuples are \textbf{ordered and grouped}. They result in different expressive power between $k$-WL and $k$-FWL.
In previous work, $k$-WL and $k$-FWL's discriminating power for \textbf{graphs} has been studied. An important result is that $k$-FWL has equal graph discriminating power to $(k+1)$-WL which is strictly stronger than $k$-WL for $k\ge2$.


For link prediction, we care about the $k=2$ case. We use $c^{(t)}(\ve)$ to denote the representation of link $\ve:=(p,q) \in [n] \times [n]$ at iteration $t$. Then, $c^{(t)}(\ve)$ in $2$-WL and $2$-FWL is updated respectively by:
\begin{align}
c^{(t)}(\ve) &= f\Big(c^{(t-1)}(\ve), \{\!\!\{c^{(t-1)}(u,q)|u\in[n]\}\!\!\}, \{\!\!\{c^{(t-1)}(p,v)|v\in[n]\}\!\!\}\Big),\label{eq:2wl}\\
c^{(t)}(\ve) &= f_F\Big(c^{(t-1)}(\ve), \{\!\!\{\big(c^{(t-1)}(u,q), c^{(t-1)}(p,u)\big)|u\in[n]\}\!\!\}\Big),\label{eq:2fwl}
\end{align}
where $f, f^F$ are injective functions. For unlabeled graphs, we can take $c^{(0)}(\ve)$ to be the indicator of whether $\ve$ exists in $E$. For labeled graphs, we additionally consider the initial node labels (features).


We can directly notice that when the initial representation for link $(p, q)$ is its 1/0 edge indicator, $2$-FWL can learn to \textbf{count the common neighbors} between $p, q$ by checking how many $(1,1)$ appear in the multiset. By iterating the third node $u$, it can actually learn all 3-node structures containing $p,q$. In contrast, $2$-WL does not learn any 3-node structure and thus cannot count common neighbors.

GNNs based on $k$-WL and $k$-FWL have been studied for graph classification. In link prediction context, however, there is no previous work that systematically characterizes $2$-WL and $2$-FWL's discriminating power for \textbf{links}. 
To compare the link discriminating power of $1$-WL and different $2$-WL variants, we first formally define \textit{$1$-WL-indistinguishable} and \textit{$2$-WL-indistinguishable}.


\begin{definition}($1$-WL-indistinguishable)\label{def:1wlindistinguishable}
Let $G = (V, E, l)$, $G' = (V', E', l')$ be two graphs, and $\vs = (s_1, s_2,...,s_k)$, $\vs' = (s_1', s_2',...,s_k')$ be two equally sized node tuples, where $s_j \in V$, $s_j' \in V' ,~\forall~j\in [k]$. Let $c^{(t)}(i)$ denote the representation of node $i$ after $t$ steps of $1$-WL update. If
\begin{align}
c^{(t)}(s_j) = c^{(t)}(s_j’), ~~\forall j \in [k],~\forall t \geq 0,
\end{align}
we say $(\vs,G)$ is $1$-WL-indistinguishable from $(\vs', G')$, denoted by $(\vs,G) \simeq_{\textrm{1-WL}} (\vs', G')$.
\end{definition}

When $\vert \vs \vert=\vert \vs'\vert=2$, we say links $\vs$ and $\vs'$ are $1$-WL-indistinguishable.
When there is a bijective mapping $\pi \in V \rightarrow V'$ such that $(s,G) \! \simeq_{\textit{1-WL}} \! (\pi(s), G'),~\forall s\in V$, it reduces to the classical graph isomorphism testing case, and we say graphs $G$ and $G'$ are $1$-WL-indistinguishable. Note that when $\vert \vs \vert < n$, we are often more concerned with the case $G=G'$, where we aim to discriminate node tuples in the same graph.


\begin{definition}($2$-WL-indistinguishable)
Given graphs $G = (V,E,l)$, $G' = (V',E',l')$ and links $\ve = (p,q)\in V \times V$, $\ve' = (p',q')\in V' \times V'$, let $c^{(t)}(\ve)$ denote the representation of $\ve$ after $t$ steps of $2$-WL update. If 
\begin{align}
c^{(t)}(\ve)=c^{(t)}(\ve'),~\forall t\ge 0,
\end{align}
we say $(\ve, G)$ is $2$-WL-indistinguishable from $(\ve', G')$, denoted by $(\ve,G) \simeq_{\textrm{2-WL}} (\ve', G')$.
\end{definition}
Similarly, we can define indistinguishable for other $2$-WL variants that take links as message passing units. Note that for $2$-WL-indistinguishable, we only consider the link case, but it is possible to generalize $2$-WL-indistinguishable to arbitrary node tuples. Notice also that a link $\ve$ is exactly a 2-tuple $\vs$ in Definition~\ref{def:1wlindistinguishable}, which allows us to compare the link discriminating power between $1$-WL and $2$-WL tests. Below we formally define the relative link discriminating power.



\begin{definition}(Discriminating Power)
Given two tests $\mathscr{A}$ and $\mathscr{B}$, if $\mathscr{A}$ distinguishes $(\ve, G)$ and $(\ve', G')$ \textbf{only if} $\mathscr{B}$ distinguishes $(\ve, G)$ and $(\ve', G')$ for any $\ve, \ve', G, G'$, and there exists some $\ve_1, \ve'_1, G_1, G'_1$ such that $(\ve_1, G_1)$ is distinguishable from $(\ve'_1, G'_1)$ by $\mathscr{B}$ but not by $\mathscr{A}$, then we say test $\mathscr{B}$ has \textbf{stronger} link discriminating power than test $\mathscr{A}$, denoted by $\mathscr{A}\prec\mathscr{B}$. If $\mathscr{A}$ distinguishes $(\ve, G)$ and $(\ve', G')$ \textbf{if and only if} $\mathscr{B}$ distinguishes $(\ve, G)$ and $(\ve', G')$ for any $\ve, \ve', G, G'$, we say test $\mathscr{A}$ has \textbf{equivalent} link discriminating power to test $\mathscr{B}$, denoted by $\mathscr{A}\sim\mathscr{B}$.
\end{definition}

Given the above definition, we are now able to compare the expressive power between $1$-WL and $2$-WL (including its variants) for link prediction. 

\section{The power of $2$-WL tests for link prediction}
In this section we theoretically characterize the link discriminating power of different $2$-WL tests by comparing them with each other and $1$-WL. We summarize our results in Table~\ref{tab:powercompare}.


\subsection{$2$-WL and $2$-FWL tests have stronger link discriminating power than $1$-WL}

\begin{wrapfigure}[14]{l}{4.5cm}
\centering
\vspace{-13pt}
\includegraphics[height=2.6cm]{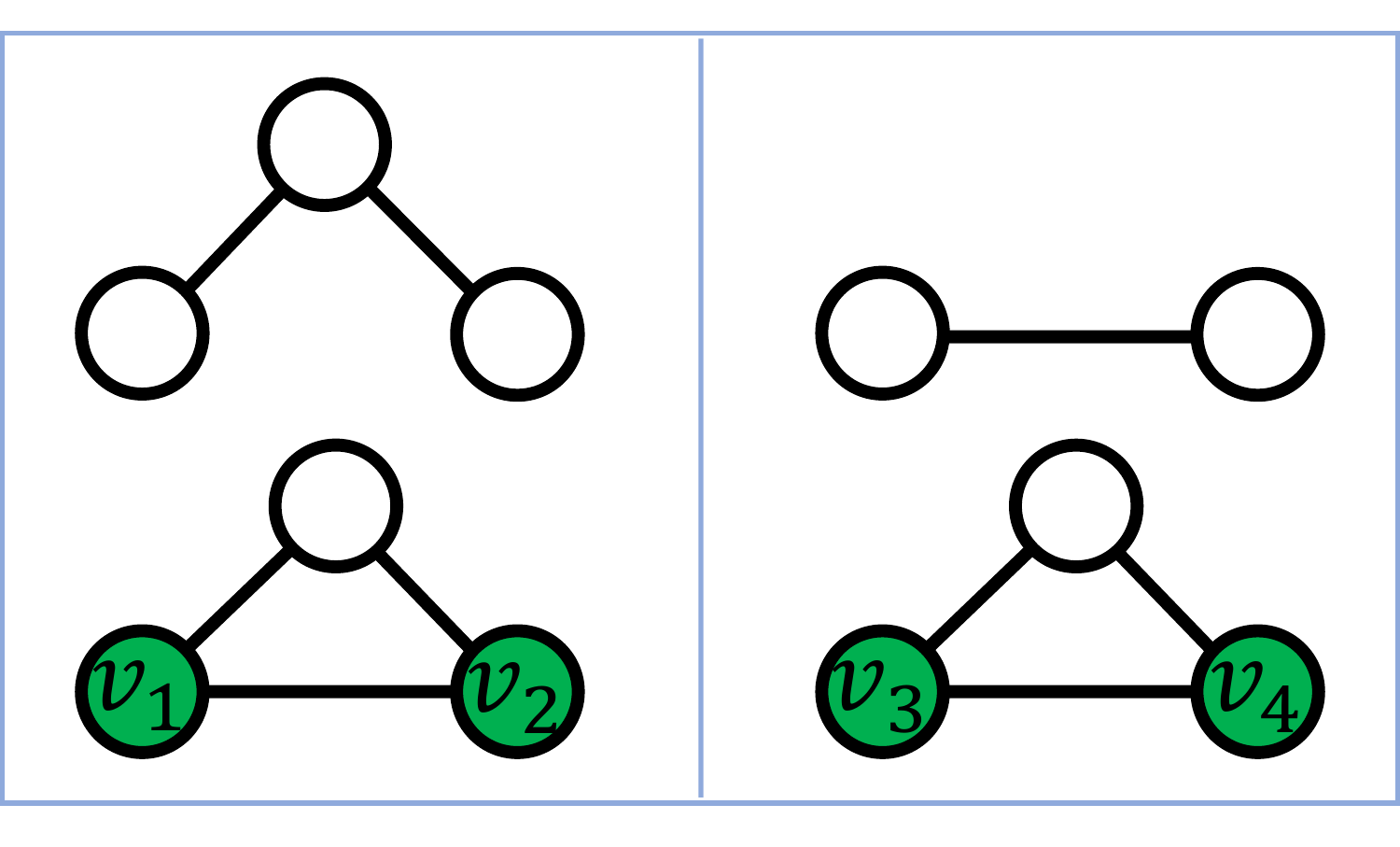}
\caption{Non-isomorphic links $(v_1,v_2)$ and $(v_3,v_4)$ from their respective graphs can be discriminated by $2$-WL but not by $1$-WL. $2$-WL can capture global features like graph size but $1$-WL only captures local structures.}
\label{fig:3}
\end{wrapfigure}


In this section, we use $2$-WL to specifically denote its plain version defined in Equation~(\ref{eq:2wl}), and $2$-FWL to denote the Folklore version defined in Equation~(\ref{eq:2fwl}). We have the following theorem. 

\begin{theorem}\label{thm:2wl}
$2$-WL has stronger link discriminating power than $1$-WL.
\end{theorem}
The theorem is immediately proved given Theorem~\ref{thm:local2wl} (which will be discussed later) and the example displayed in Figure~\ref{fig:3}. The theorem indicates that any two links that can be distinguished by $1$-WL can also be distinguished by $2$-WL, while the inverse direction is not true. It is known that $2$-WL and $1$-WL have the same \textbf{graph} discriminating power. In link prediction, however, $2$-WL is strictly stronger than $1$-WL because its neighborhood scope is global so that it can capture graph structure unconnected to the target link but $1$-WL only captures local neighborhood. However, as the two branches of neighboring links $\{\!\!\{(u,q)|u\in[n]\}\!\!\}, \{\!\!\{(p,v)|v\in[n]\}\!\!\}$ from $(p,q)$ are still independently aggregated, $2$-WL still cannot discriminate links like $(v_1,v_2)$ and $(v_1,v_3)$ in Figure~\ref{fig:1} or count common neighbors. Next, we characterize $2$-FWL's discriminating power.

\begin{theorem}\label{thm:2fwl}
$2$-FWL has stronger link discriminating power than $2$-WL.
\end{theorem}
We include the proof in the appendix. From this theorem, we can also derive that $1$-WL $\prec$ $2$-FWL.
Different from the $2$-WL case, $2$-FWL has fundamentally stronger link discriminating power than both $2$-WL and $1$-WL because it can learn three-node structures (such as common neighbors). 


\subsection{The link discriminating power of local $2$-WL and local $2$-FWL}
Compared to $1$-WL, $2$-WL and $2$-FWL have higher link discriminating power by performing message passing between high-order substructures. However, they also bring higher time and space complexity. Given a graph $G = (V, E)$ where $|V|=n$ and $|E|=m$, $1$-WL takes $O(m)$ time complexity in each iteration and occupies $O(n)$ memory. When $k\ge 2$, $k$-WL's space and time complexity grow in a polynomial rate as $O(n^k)$ and $O(kn^{k+1})$ due to storing the representations of all $n^k$ $k$-tuples and passing messages from all $kn$ neighbors for each $k$-tuple. For $2$-WL, it requires $O(n^2)$ memory and $O(n^3)$ time for each iteration, which is unaffordable for large-scale graphs.

To leverage the graph sparsity and reduce the complexity, we propose \textit{local $2$-WL}, denoted by $2$-WL$_L$. Local $2$-WL reduces the neighborhood scope in $2$-WL \textbf{from global to local}: only links that are edges in the observed graph are counted as neighbors of the current link in each iteration. The neighborhood of node pair $\ve = (p, q)$ in $2$-WL$_L$ is defined as
\begin{align}
N(\ve) = (\{\!\!\{ (u, q)~|~(u, q)\in E, u\in [n] \}\!\!\}, \{\!\!\{ (p, v)~|~(p, v)\in E, v \in [n] \}\!\!\} ).\label{eq:local2wlneighbor}
\end{align}

The following theorem characterizes local $2$-WL's discriminating power.

\begin{theorem}\label{thm:local2wl}
$2$-WL$_L$ has equivalent link discriminating power to $1$-WL.
\end{theorem}

The whole proof is included in the appendix. The main idea is to establish an bijective mapping between the subtrees of $1$-WL and those of $2$-WL$_L$. Intuitively, the two neighborhoods of $\ve= (p,q)$ in Equation~(\ref{eq:local2wlneighbor}) exactly correspond to the neighborhood of $q$ and $p$ in $1$-WL, respectively.


The above theorem establishes an interesting connection between $2$-WL and $1$-WL. We know that $2$-WL$_L$'s neighborhood is a subset of that of $2$-WL. Thus, when $2$-WL$_L$ can discriminate two links, $2$-WL can also discriminate them. Combining with the counterexample in Figure~\ref{fig:3}, we can derive Theorem~\ref{thm:2wl}. The lower power of $2$-WL$_L$ and $1$-WL than $2$-WL is rooted in their inability to detect unconnected nodes. Compared to $2$-WL with a global neighborhood definition, $2$-WL$_L$ and $1$-WL adopt local neighborhood definitions, which takes more iterations to detect long-range patterns and can never detect unconnected structures. Despite the loss of discriminating power, local $2$-WL largely reduces the time and space complexity. Denote $m'$ as the number of unknown links to predict, $2$-WL$_L$ takes $O(m+m')$ space and $O((m+m')d)$ time per iteration, where $d$ is the average node degree. 


We also propose \textit{local $2$-FWL}, denoted by $2$-FWL$_L$. Figure~\ref{fig:2fwll} gives an illustration. Given the observed graph $G=(V, E)$, we define the neighborhood of $\ve=(p, q)$ in $2$-FWL$_L$ as:
\begin{align}
N^F(\ve) = \{\!\!\{((u, q),(p, u))~|~(u, q)\in E\ \text{or}\ (p, u)\in E,~u\in [n]\}\!\!\}.
\end{align}
That is, we only keep those three-node structures $((u, q),(p, u))$ which have at least one edge existent in $G$. Therefore, $n^2$ node pairs each only need to aggregate messages from at most $2d$ three-node structures, which results in a space complexity of $O(n^2)$ and time complexity of $O(n^2d)$. Although not reducing the space complexity, $2$-FWL$_L$ significantly reduces the time complexity of $2$-FWL.



Next, we characterize the expressive power of $2$-FWL$_L$. We first compare it with $2$-WL$_L$.
\begin{theorem}\label{thm:local2fwl}
$2$-FWL$_L$ has stronger link discriminating power than $2$-WL$_L$.
\end{theorem}
The rationale is still that $2$-FWL$_L$ can count common neighbors while $2$-WL$_L$ cannot. Both $2$-WL and $2$-WL$_L$ treat their two branches of neighborhoods independently, which fails to learn the interaction between the two branches. In contrast, $2$-FWL and $2$-FWL$_L$ first group neighbor links by the shared nodes $u$, thus capturing higher-order information (e.g., three-node structures) than $2$-WL and $2$-WL$_L$.

Combining with Theorem~\ref{thm:local2wl}, we also have the conclusion that $2$-FWL$_L$ is stronger than $1$-WL for link discriminating. Furthermore, since $2$-FWL's global neighborhood is a superset of that of $2$-FWL$_L$, we have that $2$-FWL has stronger link discriminating power than $2$-FWL$_L$.

\begin{wraptable}[11]{r}{8cm}\small
    \setlength\tabcolsep{3pt}
	\centering
	\vspace{-10pt}
	\begin{tabular}{c|ccccc}
	&$1$-WL& $2$-WL$_L$&$2$-WL&$2$-FWL$_L$&$2$-FWL\\
	\hline
	$1$-WL& $\sim$&$\sim$&$\prec$&$\prec$&$\prec$\\
	$2$-WL$_L$&& $\sim$&$\prec$&$\prec$&$\prec$\\
	$2$-WL&&& $\sim$&-&$\prec$\\
	$2$-FWL$_L$&&&& $\sim$&$\prec$\\
	$2$-FWL&&&&& $\sim$\\
	\end{tabular}
	\caption{\label{tab:powercompare}The upper-triangular matrix shows relative link discriminating power of different tests, where $\sim$ denotes equal power, $\prec$ denotes weaker power, and - denotes that both are not stronger than the other.}
\end{wraptable}
Summarizing all previous results, we depict a full picture of the relative link discriminating power of all the tests in Table~\ref{tab:powercompare}. In general, the original $2$-WL tests are stronger than their local versions, and the Folklore versions ($2$-FWL and $2$-FWL$_L$) are stronger than the plain versions. All $2$-WL tests except the local $2$-WL are stronger than $1$-WL. Although the local versions are less powerful, they bring significant complexity reduction, as well as possibly more robustness and better generalizability for link prediction due to their focus on local structure patterns. Our experiments verify that local versions are usually not worse.



\section{Implementation by GNN models}\label{sec:implementation}

By implementing the injective update functions $f$ with MLPs, GNN can approach the expressive power of $1$-WL test to an arbitrary degree with enough layers and parameters~\citep{xu2018powerful}. Moreover, the update functions in GNN have learnable parameters, which allows better adaptability and generalizability. Thus, we implement our proposed $2$-WL tests for link prediction through GNNs.



\subsection{GNN implementation of $2$-WL}
For plain $2$-WL, 
we first use a $1$-WL-GNN to learn node embeddings with the raw node features, which is inspired by \citep{morris2019weisfeiler}.
If there are no raw node features, we take embeddings of node degrees to keep the inductive property of our model. Then, we obtain the initial link representations by pooling the pairwise node embeddings. 

One way to implement $2$-WL is to construct a complete graph where each node corresponds to a node pair (link) in the original graph, and then apply traditional graph convolutions. However, such an approach is unaffordable in most cases due to the $O(n^3)$ edges in the new graph. Therefore, we construct our own aggregation and combination functions. We group the link representations in the $t^{\text{th}}$ step into an $n \times n \times d$ tensor $A^{(t)}$, where the $p,q$ indexed vector $A^{(t)}_{p,q,:}$ is the representation of link $(p,q)$. For $A^{(0)}$, we also include the adjacency matrix as one slice. Then, our aggregation function and combination function are
\begin{align}
    B^{(t)}_{p,q,:} = concat\big(\sum_{i\in[n]}g(A^{(t)}_{p,i,:}),  \sum_{i\in[n]}h(A^{(t)}_{i,q,:})\big),~~~~\text{(Aggregation)}\\
    A^{(t+1)}=f\big(concat(B^{(t)}, A^{(t)})\big),~~~~\text{(Combination)}
\end{align}
where $f, g, h$ are MLPs. 
Given the ordered node pair $(p, q)$, in each layer we apply two distinct transformations $g$ and $h$ to respectively aggregate its neighbors $\{\!\!\{(u,q)|u\in[n]\}\!\!\}$,  $\{\!\!\{(p,v)|v\in[n]\}\!\!\}$. Directly operating on the dense link representations $A$ saves us from explicitly constructing the complete graph, and allows using standard GPU-based batch matrix multiplication to implement our graph convolution.
In the last layer, we pool $(p, q)$ and $(q, p)$'s representations to obtain the representation for the undirected link $\{p, q\}$.

\subsection{GNN implementation of $2$-WL$_L$}
Local $2$-WL is realised differently from $2$-WL. Due to the reduced neighborhood, we can leverage the graph sparsity to save memory and time. In each training episode, let $S$ be the mini-batch containing all positive and negative target links to predict, $E'$ be the existing edges in the original graph (after removing the positive training links). Then we construct a second-order graph $G_S := (E' \cup S, E^{(2)})$, where $E'$ and $S$ become nodes and $E^{(2)}$ denotes the edges between $E' \cup S$ based on the neighborhood definition of $2$-WL$_L$. We then apply a $1$-WL-GNN on the second-order graph $G_S$ to obtain node representations for $S$ which are used to output their link prediction scores in the original graph.
The second-order graph has $O( (|E'| + |S|) d)$ edges, where $d$ is the average node degree in the original graph. Therefore, the time complexity of message passing follows to be $O( (|E'| + |S|) d)$.
Memory efficiency is also largely improved because we only need to save $O( |E’| + |S|)$ representations.

\subsection{GNN implementation of $2$-FWL}
The situation becomes a bit more complex for $2$-FWL. The join of two links is difficult to implement by standard graph convolution layers. Thus, we apply a model similar to that proposed in \citet{maron2019provably}. In each layer, we apply slice-wise matrix multiplication of two reshaped link representation tensors to implement the $2$-FWL message passing. 
\begin{align}\label{eq:2fwlgnn}
    B^{(t)}_{p,q,:} = \sum_{i\in[n]}g(A^{(t)}_{p,i,:}) \odot h(A^{(t)}_{i,q,:}),
\end{align}
where $\odot$ is element-wise product and $g,h$ are MLPs with the same output dimension. The above implementation first joins link representations of $(p,i)$ and $(i,q)$ through element-wise product, and then performs the aggregation through summing. 
Intuitively, matrix multiplication of two adjacency matrices $AA^T$ recovers the common neighbor matrix. We also add adjacency matrix into $A^{(0)}$.


\subsection{GNN implementation of $2$-FWL$_L$}
For local $2$-FWL, we replace the dense matrix multiplications in Equation~(\ref{eq:2fwlgnn}) with sparse matrix multiplications, i.e., initially only those entries $A^{(0)}_{p,q,:}$ corresponding to existing edges $(p,q) \in E$ have nonzero values, and at the $t^{\text{th}}$ message passing step we still only track those $p,q$ entries reachable from each other in $t$ steps of random walk. Note that this implementation slightly loses the representation power because we do not learn representations for all (intermediate) links. Thus, we concatenate the final link representations with node-pair representations learned by a $1$-WL-GNN to give a nonzero representation to any link. Although this implementation does not preserve the full representation power of $2$-FWL$_L$, it can still learn common neighbor and path-counting features between nodes, and most importantly, it significantly reduces the space complexity.


\section{Related Work}
Weisfeiler-Lehman tests are a family of algorithms to deal with the graph isomorphism problem~\citep{cai1992optimal}. In addition to graph isomorphism checking, they have found many applications in machine learning recently~\citep{morris2021weisfeiler}. \citet{shervashidze2011weisfeiler} use the idea to construct subtree-based graph kernels. \citet{niepert2016learning} and \citep{zhang2017weisfeiler} use WL to sort nodes and construct neural networks for graphs.
Vanilla GNNs have also been shown to have limited graph discriminating power bounded by $1$-WL~\citep{xu2018powerful}. Many works focus on how to improve GNNs' power by considering high-dimensional WL tests. \citet{morris2019weisfeiler} introduce GNN models simulating $2$-WL and $3$-WL tests. \citet{maron2019provably,chen2019equivalence} achieve the same graph discriminating power as $3$-WL with a $2$-FWL based model. However, these works all deal with the whole-graph representation learning problem. Little work has been done in the link prediction context. In this work, we for the first time demonstrate both the theoretical and practical power of $2$-WL-based GNNs for link prediction, therefore filling in this blank area.

In the community of using GNN models for link-oriented tasks, various techniques have been proposed to enhance their theoretical power. SEAL~\citep{zhang2018link} utilizes a distance-based node labeling trick to label the context nodes according to their relationships to the target link, which is later formalized into distance encoding~\citep{li2020distance}. \citet{zhang2021labeling} further proved that such a labeling trick brings theoretical improvement to GNNs' link discriminating power. 
However, using labeling tricks requires extracting a subgraph for each link and repeatedly applying GNN to the subgraphs, which incurs high computational complexity and prevents full-batch learning. In contrast, our models aim to still apply GNN only once to the entire graph like the traditional GAE methods, while outperforming GAE in terms of link discriminating power. NBFNet~\citep{zhu2021neural} uses a type of partial labeling trick which only labels the source node and applies a GNN to predict all links from the source node. Although it does not need to extract a subgraph for every link, it needs to apply a GNN to a large graph for each source node and suffers from low training efficiency. 
On the basis of SEAL, \citet{pan2021neural} encode a transition matrix serving as a form of pairwise encoding for each link in the subgraph. However, it still requires extracting subgraphs for all links to predict.


Given original graph $G$, line graph $L(G)$ represents the adjacency between edges. In $L(G)$, each node corresponds to a unique edge in $G$. By using node representation learning methods~\citep{kipf2016semi} on the line graph, some methods~\citep{LineGraphRel, LineGraphCommDetect, LineGraphCensNet, LineGraphLP,liu2021indigo} can utilize edge features and topology better, which have achieved outstanding performance on graph tasks like heterogeneous graph learning, community detection, graph classification, and link prediction. Using $1$-WL-GNNs on line graphs is similar to local $2$-WL. However, none of these previous works have noticed the connection between line graph and $2$-WL tests. Furthermore, more expressive variants like $2$-FWL are not studied in previous works.


\section{Experiments}

In this section, we conduct experiments to verify the effectiveness of $2$-WL-GNNs for link prediction. We test $2$-WL-GNNs based on the proposed four tests: $2$-WL, local $2$-WL ($2$-WL$_L$), $2$-FWL, and local $2$-FWL ($2$-FWL$_L$). The performance metric is area under the ROC curve (AUC). For each dataset, we run each model for 10 times and report the average performance and standard deviations. Hyperparameters include learning rate, hidden dimension, number of message passing layers, and dropout rate. Baseline results are taken from \citep{zhang2018link} and \citep{zhu2021neural}. 

The baseline methods we choose are Matrix Factorization (MF)~\citep{mnih2008probabilistic}, Node2Vec (N2V)~\citep{grover2016node2vec}, Weisfeiler-Lehman Neural Machine (WLNM)~\citep{WLNM}, TLC-GNN~\citep{TLCGNN}, $1$-WL-GNNs including VGAE~\citep{VGAE} and S-VGAE~\citep{SVGAE}, and labeling trick methods including SEAL~\citep{zhang2018link} and NBFNet~\citep{zhu2021neural}. We use eleven benchmark datasets. Three of them are citation networks with node feature information: Cora, CiteSeer and Pubmed~\citep{homodata}. The other eight datasets are: USAir, NS, PB, Yeast, C.ele, Power, Router, and E.coli from SEAL, which are networks from different domains and do not contain node features.
For each network, we randomly choose $10\%$ edges as test set and $5\%$ edges as validation set. The remaining are treated as the observed training graph. The same number of randomly sampled nonexistent links are added into each set as the negative data. The results are presented in Table~\ref{tab:AUC1} and \ref{tab:AUC2}.


\begin{table}[h]\scriptsize  
    \centering
    \caption{Performance on eight networks without node features}
    \resizebox{1\textwidth}{!}{
    \begin{tabular}{cccccc|cccc}
        \toprule
        Dataset &MF&N2V&VGAE&WLNM & SEAL & $2$-WL & $2$-WL$_{L}$ & $2$-FWL & $2$-FWL$_{L}$\\
        \midrule
        USAir & 94.08$\pm$0.80& 91.44$\pm$1.78&89.28$\pm$1.99&95.95$\pm$1.10& 97.09$\pm$0.70 & 92.86$\pm$ 1.08 & 93.48$\pm$ 0.74 & \textbf{98.10$\pm$ 0.52}  &96.06$\pm$ 0.51\\
        NS & 74.55$\pm$4.34& 91.52$\pm$1.28&94.04$\pm$1.64&98.61$\pm$0.49& 97.71$\pm$0.93 & 97.15$\pm$ 0.78 & 97.37$\pm$ 0.48 & 98.85$\pm$ 0.43 &\textbf{99.49$\pm$ 0.12}\\
        PB & 94.30$\pm$0.53& 85.79$\pm$0.78&90.70$\pm$0.53&93.49$\pm$0.47& \textbf{95.01$\pm$0.34} &  93.61$\pm$ 0.54 & 94.09$\pm$ 0.26 & 94.07$\pm$ 0.47 &94.71$\pm$ 0.54\\
        Yeast & 90.28$\pm$0.69& 93.67$\pm$0.46&93.88$\pm$0.21&95.62$\pm$0.52& 97.20$\pm$0.64 & 95.76$\pm$ 0.54 & 95.16$\pm$ 0.25 & \textbf{97.82$\pm$ 0.21} &97.44$\pm$ 0.25\\
        Cele & 85.90$\pm$1.74& 84.11$\pm$1.27&81.80$\pm$2.18& 86.18$\pm$1.72& 86.54$\pm$2.04 &  81.72$\pm$ 2.15 & 86.46$\pm$ 0.98 & \textbf{88.75$\pm$ 3.99} &88.68$\pm$ 1.34\\
        Power & 50.63$\pm$1.10& 76.22$\pm$0.92&71.20$\pm$ 1.65&84.76$\pm$0.98 & 84.18$\pm$1.82 & 74.10$\pm$ 1.90 & 81.02$\pm$ 1.25 & 72.21$\pm$ 1.16 &\textbf{85.01$\pm$1.18}\\
        Router & 78.03$\pm$1.63& 65.46$\pm$0.86&61.51$\pm 1.22$&94.41$\pm$0.88& 95.68$\pm$1.22 & 96.02$\pm$ 0.61 & \textbf{96.68$\pm$ 0.50} &95.34$\pm$ 0.79 &94.91$\pm$0.64\\
        Ecoli & 93.76$\pm$0.56& 90.82$\pm$1.49&90.81$\pm$0.63&97.21$\pm$0.27& 97.22$\pm$0.28 &  96.12$\pm$ 0.48 & 95.69$\pm$ 0.52 & \textbf{98.42$\pm$ 0.21} & 97.03$\pm$ 0.57\\
        \bottomrule
    \end{tabular}}
    \label{tab:AUC1}
\end{table}

\begin{table}[h]\small
    \centering
    \caption{Performance on citation networks with node features. OOM: Out of memory.}
    \setlength\tabcolsep{3pt}
    \resizebox{1\textwidth}{!}{
    \begin{tabular}{cccccc|cccc}
        \toprule
        Dataset & VGAE & S-VGAE & TLC-GNN & SEAL & NBFNet & $2$-WL & $2$-WL$_{L}$ & $2$-FWL & $2$-FWL$_{L}$\\
        \midrule
        Cora &91.4 &94.1& 93.4 & 93.3& 95.6& 93.15$\pm$1.17 & 94.25$\pm$ 0.75 &  \textbf{96.03$\pm$0.52} & 95.18$\pm$0.86\\
        Citeseer &90.8&94.7& 90.9 & 90.5&92.3 &94.45$\pm$1.01 & 93.25$\pm$ 1.35 & 95.28$\pm$ 0.76& \textbf{95.89$\pm $0.96}\\
        Pubmed &94.4 &96.0& 97.0 & 97.8 &98.3 & OOM & \textbf{98.66$\pm$ 0.16} & OOM & 98.46$\pm$0.19\\
        \bottomrule
    \end{tabular}}
    \label{tab:AUC2}
    \vspace{-10pt}
    
\end{table}

According to the results, our $2$-WL-GNNs achieve generally better performance than the baseline models. Specifically, the $2$-WL and $2$-WL$_L$ models perform competitively with SEAL on a large number of datasets and the $2$-FWL and $2$-FWL$_L$ models obtain overall better results than SEAL. On Cora, $2$-FWL achieves a new state-of-the-art result of 96.03. On Citeseer and Pubmed, our $2$-FWL$_L$ model achieves new state-of-the-art results of 95.89 and 98.46, while $2$-WL$_L$ achieves 98.66 on Pubmed, both outperforming the previous SoTA NBFNet. Their outstanding performance verifies the effectiveness of directly using links as message passing units to learn their representations.

Theoretically, both labeling trick methods and $2$-FWL models are more expressive than $1$-WL models like VGAE and S-VGAE, which is reflected in their performance comparisons. However, we found even $2$-WL and local $2$-WL models can sometimes outperform $1$-WL-GNNs by large margins, especially on networks without node features. This might be explained by that the direct learning of link representations and the message passing along edge adjacency might capture better edge topology than node-centered methods. Furthermore, we found that the global versions $2$-WL and $2$-FWL do not always achieve better performance than their local versions $2$-WL$_L$ and $2$-FWL$_L$, despite being theoretically more powerful. This might be because the local versions focus more on local neighborhood around links, which is proved to contain the most useful information for link prediction~\citep{zhang2018link}. Considering the significantly larger memory requirement (OOM in Pubmed), we recommend to use the local versions in most cases due to their efficiency and scalability.

\begin{wraptable}[7]{r}{7cm}\small
    \setlength\tabcolsep{3pt}
	\centering
	\vspace{-15pt}
	\caption{Inference time comparison.}
	\begin{tabular}{ccccc}
        \toprule
        Dataset & $2$-WL$_{L}$ &  $2$-FWL$_{L}$& SEAL & NBFNet \\
        \midrule
        Cora & 0.007s & 1.45s& 2.30s& 1.94s\\ 
        Citeseer & 0.006s &0.74s& 2.11s & 1.80s\\
        Pubmed & 0.05s &3.9s& 15.4s & 95s\\
        \bottomrule
    \end{tabular}
    \label{tab:time}
\end{wraptable}


Finally, we present the inference time comparison between local $2$-WL models and labeling trick methods in Table~\ref{tab:time}. We compute prediction scores for all links in the test set and record the inference time of each model. The results demonstrate that local $2$-WL based models have significantly lower inference time than labeling trick methods. This is because local $2$-WL models can predict all the target links by applying the GNN once to the entire graph, while labeling trick methods require repeatedly applying GNNs to a labeled graph for every target link or source node to predict.
In Appendix~\ref{app:kgc}, we additionally evaluate the link prediction performance on knowledge graphs to further examine our proposed $2$-WL-GNNs.



\section{Conclusions}
In this paper, we have proposed two-dimensional Weisfeiler-Lehman graph neural networks for link prediction. We first discuss the problems with the prevalent $1$-WL based models, and then demonstrate the power of using $2$-WL tests to directly obtain link representations. We theoretically characterize the link discriminating power of different $2$-WL variants, including the plain $2$-WL, local $2$-WL, $2$-FWL, and local $2$-FWL. We show that except local $2$-WL, all other tests have stronger power than $1$-WL. We further propose a series of novel GNNs implementing the $2$-WL tests. Experiments on multiple benchmark datasets show the effectiveness of $2$-WL-GNNs for link prediction.  Our code is available at \url{https://github.com/GraphPKU/2WL_link_pred}.

\bibliography{paper}

\begin{thebibliography}{43}
\providecommand{\natexlab}[1]{#1}
\providecommand{\url}[1]{\texttt{#1}}
\expandafter\ifx\csname urlstyle\endcsname\relax
  \providecommand{\doi}[1]{doi: #1}\else
  \providecommand{\doi}{doi: \begingroup \urlstyle{rm}\Url}\fi

\bibitem[Adamic \& Adar(2003)Adamic and Adar]{adamic2003friends}
Adamic, L.~A. and Adar, E.
\newblock Friends and neighbors on the web.
\newblock \emph{Social networks}, 25\penalty0 (3):\penalty0 211--230, 2003.

\bibitem[Akiba et~al.(2019)Akiba, Sano, Yanase, Ohta, and
  Koyama]{akiba2019optuna}
Akiba, T., Sano, S., Yanase, T., Ohta, T., and Koyama, M.
\newblock Optuna: A next-generation hyperparameter optimization framework.
\newblock In \emph{Proceedings of the 25th ACM SIGKDD international conference
  on knowledge discovery \& data mining}, pp.\  2623--2631, 2019.

\bibitem[Al~Hasan et~al.(2006)Al~Hasan, Chaoji, Salem, and Zaki]{al2006link}
Al~Hasan, M., Chaoji, V., Salem, S., and Zaki, M.
\newblock Link prediction using supervised learning.
\newblock In \emph{SDM06: workshop on link analysis, counter-terrorism and
  security}, volume~30, pp.\  798--805, 2006.

\bibitem[Barab{\'a}si \& Albert(1999)Barab{\'a}si and
  Albert]{barabasi1999emergence}
Barab{\'a}si, A.-L. and Albert, R.
\newblock Emergence of scaling in random networks.
\newblock \emph{science}, 286\penalty0 (5439):\penalty0 509--512, 1999.

\bibitem[Cai et~al.(1992)Cai, F{\"u}rer, and Immerman]{cai1992optimal}
Cai, J.-Y., F{\"u}rer, M., and Immerman, N.
\newblock An optimal lower bound on the number of variables for graph
  identification.
\newblock \emph{Combinatorica}, 12\penalty0 (4):\penalty0 389--410, 1992.

\bibitem[Cai et~al.(2021)Cai, Li, Wang, and Ji]{LineGraphLP}
Cai, L., Li, J., Wang, J., and Ji, S.
\newblock Line graph neural networks for link prediction.
\newblock \emph{IEEE Transactions on Pattern Analysis and Machine
  Intelligence}, pp.\  1--1, 2021.

\bibitem[Chen et~al.(2019{\natexlab{a}})Chen, Li, and
  Bruna]{LineGraphCommDetect}
Chen, Z., Li, L., and Bruna, J.
\newblock Supervised community detection with line graph neural networks.
\newblock In \emph{International Conference on Learning Representations},
  2019{\natexlab{a}}.

\bibitem[Chen et~al.(2019{\natexlab{b}})Chen, Villar, Chen, and
  Bruna]{chen2019equivalence}
Chen, Z., Villar, S., Chen, L., and Bruna, J.
\newblock On the equivalence between graph isomorphism testing and function
  approximation with gnns.
\newblock \emph{Advances in neural information processing systems}, 32,
  2019{\natexlab{b}}.

\bibitem[Davidson et~al.(2018)Davidson, Falorsi, Cao, Kipf, and Tomczak]{SVGAE}
Davidson, T.~R., Falorsi, L., Cao, N.~D., Kipf, T., and Tomczak, J.~M.
\newblock Hyperspherical variational auto-encoders.
\newblock In \emph{Uncertainty in Artificial Intelligence}, pp.\  856--865.
  {AUAI} Press, 2018.

\bibitem[Grohe(2017)]{grohe2017descriptive}
Grohe, M.
\newblock \emph{Descriptive complexity, canonisation, and definable graph
  structure theory}, volume~47.
\newblock Cambridge University Press, 2017.

\bibitem[Grover \& Leskovec(2016)Grover and Leskovec]{grover2016node2vec}
Grover, A. and Leskovec, J.
\newblock node2vec: Scalable feature learning for networks.
\newblock In \emph{Proceedings of the 22nd ACM SIGKDD international conference
  on Knowledge discovery and data mining}, pp.\  855--864, 2016.

\bibitem[Jiang et~al.()Jiang, Ji, and Li]{LineGraphCensNet}
Jiang, X., Ji, P., and Li, S.
\newblock Censnet: Convolution with edge-node switching in graph neural
  networks.
\newblock In Kraus, S. (ed.), \emph{International Joint Conference on
  Artificial Intelligence}, pp.\  2656--2662.

\bibitem[Kipf \& Welling(2016{\natexlab{a}})Kipf and Welling]{VGAE}
Kipf, T.~N. and Welling, M.
\newblock Variational graph auto-encoders.
\newblock \emph{CoRR}, abs/1611.07308, 2016{\natexlab{a}}.
\newblock URL \url{http://arxiv.org/abs/1611.07308}.

\bibitem[Kipf \& Welling(2016{\natexlab{b}})Kipf and Welling]{kipf2016semi}
Kipf, T.~N. and Welling, M.
\newblock Semi-supervised classification with graph convolutional networks.
\newblock \emph{arXiv preprint arXiv:1609.02907}, 2016{\natexlab{b}}.

\bibitem[Kipf \& Welling(2016{\natexlab{c}})Kipf and
  Welling]{kipf2016variational}
Kipf, T.~N. and Welling, M.
\newblock Variational graph auto-encoders.
\newblock \emph{arXiv preprint arXiv:1611.07308}, 2016{\natexlab{c}}.

\bibitem[Li et~al.(2020)Li, Wang, Wang, and Leskovec]{li2020distance}
Li, P., Wang, Y., Wang, H., and Leskovec, J.
\newblock Distance encoding: Design provably more powerful neural networks for
  graph representation learning.
\newblock \emph{Advances in Neural Information Processing Systems},
  33:\penalty0 4465--4478, 2020.

\bibitem[Liben-Nowell \& Kleinberg(2007)Liben-Nowell and
  Kleinberg]{liben2007link}
Liben-Nowell, D. and Kleinberg, J.
\newblock The link-prediction problem for social networks.
\newblock \emph{Journal of the American society for information science and
  technology}, 58\penalty0 (7):\penalty0 1019--1031, 2007.

\bibitem[Liu et~al.(2021)Liu, Grau, Horrocks, and Kostylev]{liu2021indigo}
Liu, S., Grau, B., Horrocks, I., and Kostylev, E.
\newblock Indigo: Gnn-based inductive knowledge graph completion using
  pair-wise encoding.
\newblock \emph{Advances in Neural Information Processing Systems}, 34, 2021.

\bibitem[Maron et~al.(2019)Maron, Ben-Hamu, Serviansky, and
  Lipman]{maron2019provably}
Maron, H., Ben-Hamu, H., Serviansky, H., and Lipman, Y.
\newblock Provably powerful graph networks.
\newblock \emph{Advances in neural information processing systems}, 32, 2019.

\bibitem[Menon \& Elkan(2011)Menon and Elkan]{menon2011link}
Menon, A.~K. and Elkan, C.
\newblock Link prediction via matrix factorization.
\newblock In \emph{Joint european conference on machine learning and knowledge
  discovery in databases}, pp.\  437--452. Springer, 2011.

\bibitem[Mnih \& Salakhutdinov(2008)Mnih and
  Salakhutdinov]{mnih2008probabilistic}
Mnih, A. and Salakhutdinov, R.~R.
\newblock Probabilistic matrix factorization.
\newblock In \emph{Advances in neural information processing systems}, pp.\
  1257--1264, 2008.

\bibitem[Morris et~al.(2019)Morris, Ritzert, Fey, Hamilton, Lenssen, Rattan,
  and Grohe]{morris2019weisfeiler}
Morris, C., Ritzert, M., Fey, M., Hamilton, W.~L., Lenssen, J.~E., Rattan, G.,
  and Grohe, M.
\newblock Weisfeiler and leman go neural: Higher-order graph neural networks.
\newblock In \emph{Proceedings of the AAAI conference on artificial
  intelligence}, volume~33, pp.\  4602--4609, 2019.

\bibitem[Morris et~al.(2021)Morris, Lipman, Maron, Rieck, Kriege, Grohe, Fey,
  and Borgwardt]{morris2021weisfeiler}
Morris, C., Lipman, Y., Maron, H., Rieck, B., Kriege, N.~M., Grohe, M., Fey,
  M., and Borgwardt, K.
\newblock Weisfeiler and leman go machine learning: The story so far.
\newblock \emph{arXiv preprint arXiv:2112.09992}, 2021.

\bibitem[Nickel et~al.(2015)Nickel, Murphy, Tresp, and
  Gabrilovich]{nickel2015review}
Nickel, M., Murphy, K., Tresp, V., and Gabrilovich, E.
\newblock A review of relational machine learning for knowledge graphs.
\newblock \emph{Proceedings of the IEEE}, 104\penalty0 (1):\penalty0 11--33,
  2015.

\bibitem[Niepert et~al.(2016)Niepert, Ahmed, and Kutzkov]{niepert2016learning}
Niepert, M., Ahmed, M., and Kutzkov, K.
\newblock Learning convolutional neural networks for graphs.
\newblock In \emph{International conference on machine learning}, pp.\
  2014--2023. PMLR, 2016.

\bibitem[Oyetunde et~al.(2017)Oyetunde, Zhang, Chen, Tang, and
  Lo]{oyetunde2017boostgapfill}
Oyetunde, T., Zhang, M., Chen, Y., Tang, Y., and Lo, C.
\newblock Boostgapfill: improving the fidelity of metabolic network
  reconstructions through integrated constraint and pattern-based methods.
\newblock \emph{Bioinformatics}, 33\penalty0 (4):\penalty0 608--611, 2017.

\bibitem[Pan et~al.(2021)Pan, Shi, and Dokmani{\'c}]{pan2021neural}
Pan, L., Shi, C., and Dokmani{\'c}, I.
\newblock Neural link prediction with walk pooling.
\newblock \emph{arXiv preprint arXiv:2110.04375}, 2021.

\bibitem[Schlichtkrull et~al.(2017)Schlichtkrull, Kipf, Bloem, Berg, Titov, and
  Welling]{schlichtkrull2017modeling}
Schlichtkrull, M., Kipf, T.~N., Bloem, P., Berg, R. v.~d., Titov, I., and
  Welling, M.
\newblock Modeling relational data with graph convolutional networks.
\newblock \emph{arXiv preprint arXiv:1703.06103}, 2017.

\bibitem[Sen et~al.(2008)Sen, Namata, Bilgic, Getoor, Galligher, and
  Eliassi-Rad]{homodata}
Sen, P., Namata, G., Bilgic, M., Getoor, L., Galligher, B., and Eliassi-Rad, T.
\newblock Collective classification in network data.
\newblock \emph{AI magazine}, 29\penalty0 (3):\penalty0 93--93, 2008.

\bibitem[Shervashidze et~al.(2011)Shervashidze, Schweitzer, Van~Leeuwen,
  Mehlhorn, and Borgwardt]{shervashidze2011weisfeiler}
Shervashidze, N., Schweitzer, P., Van~Leeuwen, E.~J., Mehlhorn, K., and
  Borgwardt, K.~M.
\newblock Weisfeiler-lehman graph kernels.
\newblock \emph{Journal of Machine Learning Research}, 12\penalty0 (9), 2011.

\bibitem[Teru et~al.(2020)Teru, Denis, and Hamilton]{teru2020inductive}
Teru, K., Denis, E., and Hamilton, W.
\newblock Inductive relation prediction by subgraph reasoning.
\newblock In \emph{International Conference on Machine Learning}, pp.\
  9448--9457. PMLR, 2020.

\bibitem[Trouillon et~al.(2016)Trouillon, Welbl, Riedel, Gaussier, and
  Bouchard]{trouillon2016complex}
Trouillon, T., Welbl, J., Riedel, S., Gaussier, {\'E}., and Bouchard, G.
\newblock Complex embeddings for simple link prediction.
\newblock In \emph{International conference on machine learning}, pp.\
  2071--2080. PMLR, 2016.

\bibitem[Xu et~al.(2018)Xu, Hu, Leskovec, and Jegelka]{xu2018powerful}
Xu, K., Hu, W., Leskovec, J., and Jegelka, S.
\newblock How powerful are graph neural networks?
\newblock \emph{arXiv preprint arXiv:1810.00826}, 2018.

\bibitem[Yan et~al.(2021)Yan, Ma, Gao, Tang, and Chen]{TLCGNN}
Yan, Z., Ma, T., Gao, L., Tang, Z., and Chen, C.
\newblock Link prediction with persistent homology: An interactive view.
\newblock In \emph{International Conference on Machine Learning}, volume 139,
  pp.\  11659--11669. {PMLR}, 2021.

\bibitem[You et~al.(2021)You, Gomes-Selman, Ying, and
  Leskovec]{you2021identity}
You, J., Gomes-Selman, J., Ying, R., and Leskovec, J.
\newblock Identity-aware graph neural networks.
\newblock \emph{arXiv preprint arXiv:2101.10320}, 2021.

\bibitem[Zhang \& Chen(2017{\natexlab{a}})Zhang and Chen]{WLNM}
Zhang, M. and Chen, Y.
\newblock Weisfeiler-lehman neural machine for link prediction.
\newblock In \emph{International Conference on Knowledge Discovery and Data
  Mining}, pp.\  575--583. {ACM}, 2017{\natexlab{a}}.

\bibitem[Zhang \& Chen(2017{\natexlab{b}})Zhang and Chen]{zhang2017weisfeiler}
Zhang, M. and Chen, Y.
\newblock Weisfeiler-lehman neural machine for link prediction.
\newblock In \emph{Proceedings of the 23rd ACM SIGKDD International Conference
  on Knowledge Discovery and Data Mining}, pp.\  575--583, 2017{\natexlab{b}}.

\bibitem[Zhang \& Chen(2018)Zhang and Chen]{zhang2018link}
Zhang, M. and Chen, Y.
\newblock Link prediction based on graph neural networks.
\newblock \emph{Advances in neural information processing systems}, 31, 2018.

\bibitem[Zhang \& Chen(2020)Zhang and Chen]{Zhang2020Inductive}
Zhang, M. and Chen, Y.
\newblock Inductive matrix completion based on graph neural networks.
\newblock In \emph{International Conference on Learning Representations}, 2020.
\newblock URL \url{https://openreview.net/forum?id=ByxxgCEYDS}.

\bibitem[Zhang et~al.(2021)Zhang, Li, Xia, Wang, and Jin]{zhang2021labeling}
Zhang, M., Li, P., Xia, Y., Wang, K., and Jin, L.
\newblock Labeling trick: A theory of using graph neural networks for
  multi-node representation learning.
\newblock \emph{Advances in Neural Information Processing Systems}, 34, 2021.

\bibitem[Zhou et~al.(2009)Zhou, L{\"u}, and Zhang]{zhou2009predicting}
Zhou, T., L{\"u}, L., and Zhang, Y.-C.
\newblock Predicting missing links via local information.
\newblock \emph{The European Physical Journal B}, 71\penalty0 (4):\penalty0
  623--630, 2009.

\bibitem[Zhu et~al.(2019)Zhu, Zhou, Pan, Zhu, and Wang]{LineGraphRel}
Zhu, S., Zhou, C., Pan, S., Zhu, X., and Wang, B.
\newblock Relation structure-aware heterogeneous graph neural network.
\newblock In \emph{{IEEE} International Conference on Data Mining}, pp.\
  1534--1539. {IEEE}, 2019.

\bibitem[Zhu et~al.(2021)Zhu, Zhang, Xhonneux, and Tang]{zhu2021neural}
Zhu, Z., Zhang, Z., Xhonneux, L.-P., and Tang, J.
\newblock Neural bellman-ford networks: A general graph neural network
  framework for link prediction.
\newblock \emph{Advances in Neural Information Processing Systems}, 34, 2021.

\end{thebibliography}
\bibliographystyle{icml2022}

\appendix

\newpage
\section{Proof of Theorem~\ref{thm:local2wl}}

\textbf{Theorem}:
$2$-WL$_L$ has the same discriminating power as $1$-WL for link prediction.

\begin{proof}
This and all of other theorems have a presumption that whether the target links are connected is unknown. Their indicator of edge existence is set to zero, otherwise the series of $2$-WL tests can directly give the correct prediction. This is necessary for the link prediction context. 

We measure the link discriminating power by constructing subtrees. Given an undirected graph $G=(V,E,l),~p, q\in V$, let $T_{\mathscr{A}}$, $T_{\mathscr{B}}$ be mappings from sets of graph-link tuples $(G,(p,q))$ to sets of tree-structured graphs with infinite depth, which are defined as follows.

For graph $G$, $p, q\in G$, $|G|= n$. $T_{\mathscr{A}}(G,(p,q))$ has a root $(p,q)$ labeled as $(l(p),l(q))$ with two branches of child nodes $\{(p,i):(p,i)\in E, i\in [n]\}$ and $\{(j,q):(j,q)\in E, j\in [n]\}$ in the left and right side, respectively. For every child node $(r,s)$, it is labeled as $(l(r),l(s))$. Its child nodes and their labels are defined in the same way recursively.

$T_{\mathscr{B}}(G,(p,q))$ has a root $(p,q)$ which is labeled as $(l(p),l(q))$. It has two branches of child nodes: $\{i:(p,i)\in E,i\in[n]\}$ on the left and $\{j:(j,q)\in E,j\in[n]\}$ on the right. In the following layers node $k$ has children $\{l:(k,l)\in E,l\in[n]\}$.  Node $k$ is labeled in the graph as $l(k)$.

Then we define an equivalent class across the tree-structured graphs: Denote $E_T=\{(prec, next, br):next \textrm{ is the child node of }prec\textrm{ in }\textrm{branch } br \textrm{ in tree }T\}$. If there is a bijective mapping $\pi$ from nodes of a finite-depth tree $T_1$ (denoted by $V(T_1)$) to nodes of a finite-depth tree $T_2$ (denoted by $V(T_2)$) such that 1) $l(i)=l(\pi(i)),\forall i\in V(T_1)$, 2) $(i,j,br)\in E_{T_1} \iff (\pi(i),\pi(j),br)\in E_{T_2},\forall i,j\in V(T_1)$, we say $T_1$ is equivalent to $T_2$, denoted as $T_1\simeq T_2$.

Let $T\vert_k$ refer to the mapping that $T\vert_k(G,e)$ is the first $k$ layers of subtree $T(G,e)$. We define that two infinite-depth trees $T_1$, $T_2$ satisfy $T_1\simeq T_2$ if and only if $T_1\vert_k\simeq T_2\vert_k, \forall k\in \mathbb{N}$.

Given the well defined equivalent class and Definition~\ref{def:1wlindistinguishable}, we notice that $T_{\mathscr{A}}$, $T_{\mathscr{B}}$ depict the process of local $2$-WL and $1$-WL test, that is, \begin{align}
    ((p,q),G)\simeq_{\textrm{2-WL$_L$}}((p',q'),G')\iff
    T_{\mathscr{A}}(G,(p,q))\simeq T_{\mathscr{A}}(G',(p',q'))\\
    ((p,q),G)\simeq_{\textrm{1-WL}}((p',q'),G')\iff
    T_{\mathscr{B}}(G,(p,q))\simeq T_{\mathscr{B}}(G',(p',q'))
\end{align}
Therefore the statement that local $2$-WL and $1$-WL has equivalent link discriminating power equals to that $\forall (G, e), (G', e'),$
\begin{align}
T_{\mathscr{A}}(G,e)\simeq T_{\mathscr{A}}(G',e') \iff T_{\mathscr{B}}(G,e)\simeq T_{\mathscr{B}}(G',e').
\end{align}

According to our definition, we need to prove that for $\forall k\in \mathbb{N}$,
\begin{align}
T_{\mathscr{A}}\vert_k(G,e)\simeq T_{\mathscr{A}}\vert_k(G',e') \iff T_{\mathscr{B}}\vert_k(G,e)\simeq T_{\mathscr{B}}\vert_k(G',e'),~\forall (G,e),(G',e')
\end{align}
For $k=0$, since $1(e), 1(e')$ are unknown,  
we have 
$T_{\mathscr{A}}\vert_0(G,e)\simeq T_{\mathscr{A}}\vert_0(G',e')
\iff l(p)=l(p'),~l(q)=l(q')
\iff T_{\mathscr{B}}\vert_0(G,e)\simeq T_{\mathscr{B}}\vert_0(G',e')$

Suppose $(15)$ works for $k=L$, let's consider the situation of $k=L+1$:

Denote $\{i_1,i_2,...,i_{n_p}\}$, $\{j_1,j_2,...,j_{n_q}\}$ as neighbors of $p,q$ in $G$, and $\{i'_1,i'_2,...,i'_{n_{p'}}\}$, $\{j'_1,j'_2,...,j'_{n_{q'}}\}$ as neighbors of $p',q'$ in $G'$, respectively. 
According to the property of local $2$-WL test, if $T_{\mathscr{A}}\vert_{L+1}(G,e)\simeq T_{\mathscr{A}}\vert_{L+1}(G',e')$, we have $l(p)=l(p'),~l(q)=l(q'),~n_p = n_{p'},~n_q=n_{q'}$ and w.l.o.g.
\begin{align}
T_{\mathscr{A}}\vert_L(G,(p,i_s))\simeq T_{\mathscr{A}}\vert_L(G',(p',i'_s)),~\forall s\in[n_p]\\
T_{\mathscr{A}}\vert_L(G,(j_t,q))\simeq T_{\mathscr{A}}\vert_L(G',(j'_t,q')),~\forall t\in[n_q]
\end{align}

Since $(15)$ works for $k=L$, we have
\begin{align}
T_{\mathscr{B}}\vert_L(G,(p,i_s))\simeq T_{\mathscr{B}}\vert_L(G',(p',i'_s)),~\forall s\in[n_p]\\
T_{\mathscr{B}}\vert_L(G,(j_t,q))\simeq T_{\mathscr{B}}\vert_L(G',(j'_t,q')),~\forall t\in[n_q]
\end{align}

According to property of $1$-WL test, $(18),(19)$ mean $T_{\mathscr{B}}\vert_{L+1}(G,(p,q))$ and $T_{\mathscr{B}}\vert_{L+1}(G',(p',q'))$ have $m+n$ correspondingly isomorphic $L$-depth branches. Plus $l(p)=l(p'),~l(q)=l(q')$ we conclude $T_{\mathscr{B}}\vert_{L+1}(G,(p,q))\simeq T_{\mathscr{B}}\vert_{L+1}(G',(p',q'))$.
The other direction can be similarly proved. Figure~\ref{fig:subtree} gives an illustration.
\end{proof}

\begin{figure}[t]
\centering
\includegraphics[height=5.5cm]{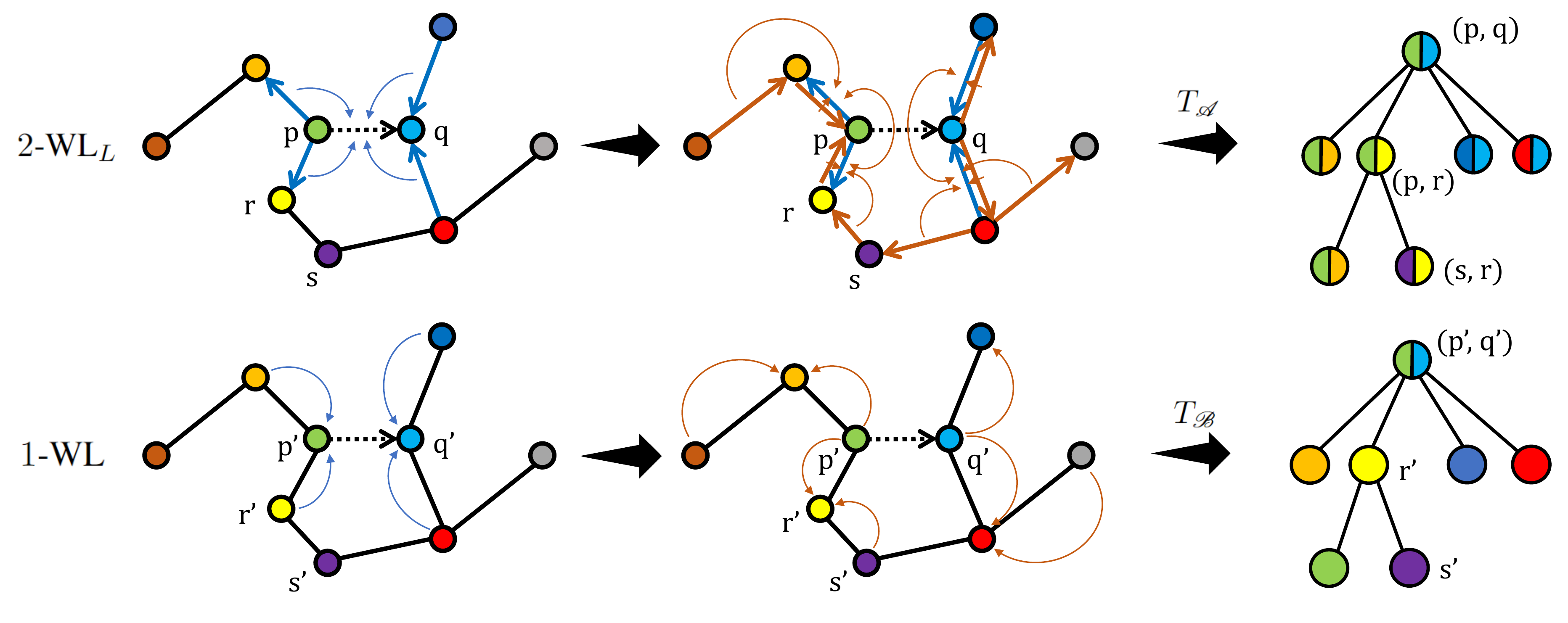}
\caption{
Update patterns of $2$-WL$_L$ and $1$-WL test, and their corresponding mappings from $(G, e)$ to subtrees in our proof. We can build a one-to-one mapping between subtrees of $2$-WL$_L$ and $1$-WL. We show only part of subtrees.
}
\label{fig:subtree}
\end{figure}


\section{Proof of Theorem~\ref{thm:2fwl} and Theorem~\ref{thm:local2fwl}}

\textbf{Theorem:}
$2$-FWL has stronger link discriminating power than $2$-WL.

\begin{proof}
Let $T_{\mathscr{C}}$, $T_{\mathscr{D}}$ be mappings from sets of graph-link tuples to sets of tree-structured graphs with infinite depth.

For graph $G=(V,E,l)$, $p, q\in V$, $|G|=n$. $T_{\mathscr{C}}(G,(p,q))$ has a root $(p,q)$ labeled as $(l(p),l(q))$ with two branches of child nodes $\{(p,i):i\in [n]\}$ and $\{(j,q):j\in [n]\}$ on the left and right side, respectively. For every child node $(r,s)$, its child nodes are defined in the same way recursively. Node $(r,s)$ is labeled as $(l(r),l(s),1_{\{(r,s)\in E\}},1_{\{r=s\}})$.

$T_{\mathscr{D}}(G,(p,q))$ has root $(p,q)$ which is labeled as $(l(p),l(q))$. It has $n$ child nodes $\{((p,i),(i,q))\vert i\in[n]\}$. Node $((p,r),(r,q))$ is labeled as $(l(p),l(r),l(q),1_{\{(p,r)\in E\}},1_{\{(r,q)\in E\}},1_{\{p=r\}},1_{\{r=q\}})$ which has two branches of child nodes $\{((p,t),(t,r))\vert t\in[n]\}$ and $\{((r,s),(s,q))\vert s\in[n]\}$. Each of them has its label and child nodes defined in the same way recursively.

Therefore $T_{\mathscr{C}}$ and $T_{\mathscr{D}}$ depict the process of $2$-WL and $2$-FWL tests. After defining the equivalent class of tree-structured graph as in the proof of Theorem~\ref{thm:local2wl}, we have
\begin{align}
    (G,(p,q))\simeq_{\textrm{2-WL}}(G',(p',q'))\iff
    T_{\mathscr{C}}(G,(p,q))\simeq T_{\mathscr{C}}(G',(p',q'))\\
    (G,(p,q))\simeq_{\textrm{2-FWL}}(G',(p',q'))\iff
    T_{\mathscr{D}}(G,(p,q))\simeq T_{\mathscr{D}}(G',(p',q'))
\end{align}
Let $T\vert_k$ refer to the mapping such that $T\vert_k(G,e)$ is the first $k$ layers of subtree $T(G,e)$.
We will prove that for $\forall k\in\mathbb{N}$,
\begin{align}
T_{\mathscr{D}}\vert_k(G,e)\simeq T_{\mathscr{D}}\vert_k(G',e') \Rightarrow T_{\mathscr{C}}\vert_k(G,e)\simeq T_{\mathscr{C}}\vert_k(G',e'),~\forall (G,e),(G',e')
\label{eq:22}
\end{align}
Fix $(G,e),(G',e'),G=(V,E,l),G'=(V',E',l')$. Let $n=\vert V\vert, n'=\vert V'\vert$. When $k=0$, $T_{\mathscr{D}}\vert_0(G,e)\simeq T_{\mathscr{D}}\vert_0(G',e')
\Rightarrow l(p)=l(p'),~l(q)=l(q')
\Rightarrow T_{\mathscr{C}}\vert_0(G,e)\simeq T_{\mathscr{C}}\vert_0(G',e')$

Suppose (\ref{eq:22}) is true for $k=L, L\ge0$. Let's consider the situation of $k=L+1$.
According to the property of $2$-FWL test, if $T_{\mathscr{D}}\vert_{L+1}(G,e)\simeq T_{\mathscr{D}}\vert_{L+1}(G',e')$, we immediately have $n=n'$ and w.l.o.g.
\begin{align}
&l(i)=l(i'),~\forall i\in[n]\\
&1_{\{(p,i)\in E\}}=1_{\{(p',i')\in E'\}},~\forall i\in[n]\\
&1_{\{(i,q)\in E\}}=1_{\{(i',q')\in E'\}},~\forall i\in[n]\\
(T_{\mathscr{D}}\vert_L(G,(p,i)),T_{\mathscr{D}}\vert_L(&G,(i,q)))\simeq(T_{\mathscr{D}}\vert_L(G',(p',i')),T_{\mathscr{D}}\vert_L(G',(i',q'))),~\forall i\in[n]
\end{align}
Then we have
\begin{align}
T_{\mathscr{D}}\vert_L(G,(p,i))\simeq T_{\mathscr{D}}\vert_L(G',(p',i')),~\forall i\in[n]\\
T_{\mathscr{D}}\vert_L(G,(j,q))\simeq T_{\mathscr{D}}\vert_L(G',(j',q')),~\forall j\in[n]
\end{align}
Due to that (\ref{eq:22}) is true for $k=L$, we have
\begin{align}
T_{\mathscr{C}}\vert_L(G,(p,i))\simeq T_{\mathscr{C}}\vert_L(G',(p',i')),~\forall i\in[n]\\
T_{\mathscr{C}}\vert_L(G,(j,q))\simeq T_{\mathscr{C}}\vert_L(G',(j',q')),~\forall j\in[n]
\end{align}
According to (23), (24), (25), (29), (30) and the definition of $T_{\mathscr{C}}$, we have
\begin{align}
T_{\mathscr{C}}\vert_{L+1}(G,(p,q))\simeq T_{\mathscr{C}}\vert_{L+1}(G',(p',q')),~\forall i\in[n]
\end{align}
On the other side the counterexample lies in Figure~\ref{fig:4}
\end{proof}

\textbf{Theorem:}
$2$-FWL$_L$ has stronger link discriminating power than $2$-WL$_L$.
\begin{proof}
The proof is the same as the proof of Theorem~\ref{thm:2fwl} except that (23)-(30) works for $p,q$ and their neighbors instead of all nodes.
\end{proof}
\begin{figure}[t]
\centering
\includegraphics[height=3.0cm]{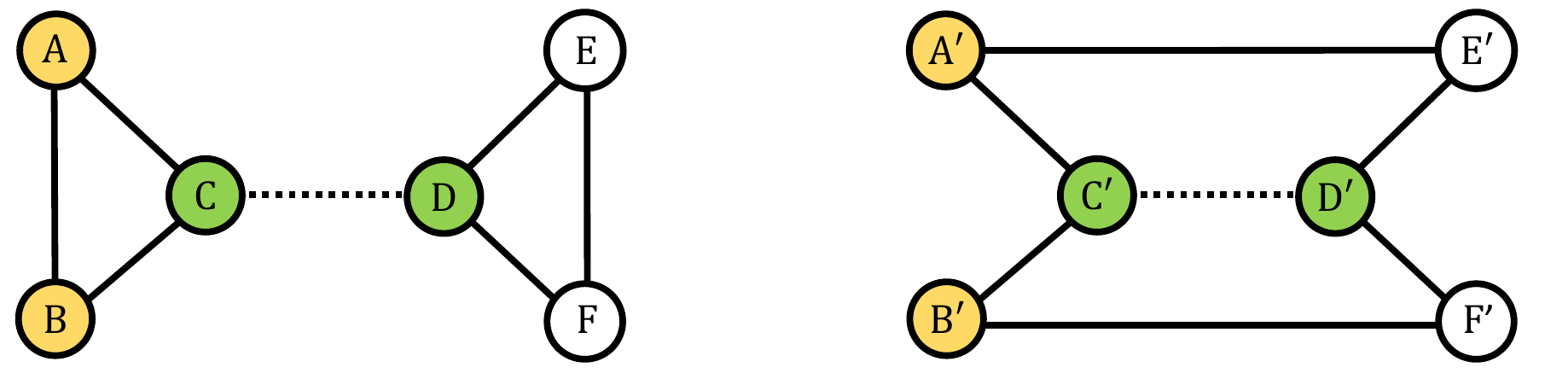}
\caption{This figure contains two counterexample. First, links $(A,B)$ and $(A',E')$ cannot be distinguished by plain $2$-WL but can be distinguished by (local) $2$-FWL and $1$-WL with 0/1 labeling trick. In fact due to high node-level symmetry $2$-WL cannot detect difference between any connected link pairs or unconnected link pairs. The labeling trick breaks such symmetry and help $1$-WL to capture the difference of two graph's structure.
If $(C,D)$ and $(C',D')$ are target node pairs, 0/1 labeling trick no longer works. However, $2$-FWL and local $2$-FWL still work (because $(D,E)$ and $(D',E')$ will have different representations). They can capture triple structure as $3$-WL test does.}
\label{fig:4}
\end{figure}

\section{Extended discussion on labeling trick}
In this section, we compare the link discriminating power between $2$-WL tests and $1$-WL with labeling tricks.
There are two most classic labeling tricks for link prediction: the 0/1 labeling and distance-based labeling, the former labels the target nodes pair with one and other nodes with zero. A classic instance of distance-based labeling is DRNL (Double-Radius Node Labeling) in \citep{zhang2018link}. It constructs an injective function of distances from current node to two target nodes. Such a technique inherently makes use of the information of all paths to the target nodes within the extracted subgraph, which itself is a strong heuristic of link prediction. Note that node labeling can be directly included in label (feature) $l$.  

Here we mainly discuss 0/1 labeling in the following and leave the discussion on distance-based labeling tricks and more general ones to the future work. \citet{zhang2021labeling} has discussed the theoretical power of 0/1 labeling trick and showed that it enhances $1$-WL's link discriminating power. We further compare the link discriminating power of $1$-WL with labeling trick and $2$-WL tests in the following theorem:


\begin{theorem}\label{label2}
For 0/1 labeling trick $L$, $1$-WL test with $L$ and local $2$-FWL test both do not have equal or stronger link discriminating power than the other.
\end{theorem}
\begin{proof}
 $(C,D),~(C',D')$ in Figure~\ref{fig:4} present an example that $2$-FWL$_L$ can discriminate but $1$-WL with $L$ cannot.
  On the other hand, let's consider 4-order magic square graphs. Below are two $4\times4$ grid graphs without node features. Each node has a number from $\{1,~2,~3,~4\}$ on it. Two nodes have edges if and only if they are 1) in the same row, or 2) in the same column, or 3) holding the same number. The colored node pair $(p,q),~(p',q')$ are the target links. Notice that they are both strongly regular graphs and $2$-FWL cannot discriminate the two links because any node pair with edge has two common neighbors and any node pair without edge also has two common neighbors.
 \begin{figure}[h]
\centering
\includegraphics[height=3.5cm]{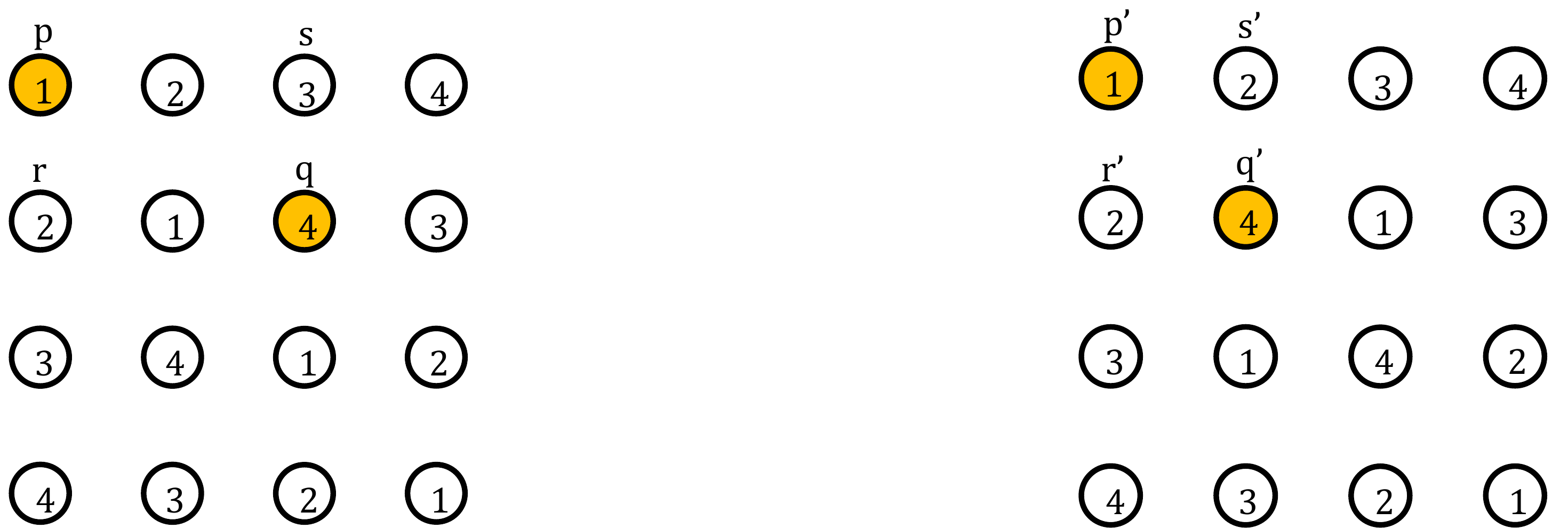}
\label{fig:5}
\end{figure}

 $1$-WL with 0/1 labeling can discriminate $(p,q),~(p',q')$. If not, $(r,s),~(r',s')$ (or $(r,s),~(s',r')$) will be indistinguishable from each other but distinguishable from other node pairs because they are the only nodes that have two labeled children. However they can actually be discriminated since $(r,s)$ does not have edge but $(r',s')$ does, which leads to a contradiction.
\end{proof}

\section{More details on GNN implementations} \label{app:2fwlimplementation}
\textbf{Computing infrastructure.} We leverage Pytorch Geometric V2.0.2 and Pytorch V1.10.0 for model development.
We train our models, measure AUROC and the inference time on
an A40 GPU with 48GB memory on a Linux server.

\textbf{Implementation of Message Passing Networks.} Raw node features are used for initial node embeddings. If there are no raw node features, we take embeddings of node degrees. Then we use $1$-WL-GNN layers to deal with node embeddings.\\
\textbf{$2$-WL:} Pool node embeddings to obtain $n^2$ link embeddings. Add an adjacency matrix as a slice of $n*n*d$ link embedding tensor and then apply (9), (10) in every $2$-WL-GNN layer.\\
\textbf{$2$-WL$_L$:} Denote observed edges as $E$ and mini-batch of target links as $S$. pool node embeddings to obtain $\vert E\vert + \vert S\vert$ link embeddings. For target link $(p,q)$, apply two different GCN layers to process links $\{(p,i)\}$, $\{(j,q)\}$ respectively and take their sum to form a whole $2$-WL-GNN layer.\\
\textbf{$2$-FWL:} Pool node embeddings to obtain $n^2$ link embeddings. Add an identity and an adjacency matrix as two slices of $n*n*d$ link embedding tensor. Apply linear layers on the third dimension and slice-wise matrix multiplication.\\
\textbf{$2$-FWL$_L$:} Denote observed edges as $E$ and mini-batch of target links as $S$. Pool node embeddings to obtain $\vert E\vert$ positive link embeddings and $\vert S\vert$ negative link embeddings. Reform positive link embeddings to sparse tensor, apply linear layers and slice-wise sparse matrix multiplication on it as one $2$-WL-GNN layer. Finally concatenate $2$-WL representation with link embeddings to obtain nonzero representations.

For more details on the implementations, please refer to our code.

\textbf{Baselines.} For AUROC of methods: MF, N2V, VGAE, WLNM, SEAL on non-featured datasets, we directly use the results in \citet{zhang2018link}. For AUROC of methods: VGAE, S-VGAE, TLC-GNN, SEAL, NBFNet on citation datasets, we directly use the results in \citet{zhu2021neural}. For performance of KGC methods: R-GCN, GraIL, INDIGO on KG datasets, we use the results in \citet{liu2021indigo}

\textbf{Hyperparameter tuning.} Hyperparameters are selected based on validation set performance. The best hyperparameters can be found in our code in the supplement material. Learning rate $lr$ is chosen from: $\{5e-2,~1e-2,~5e-3,~1e-3,~5e-4\}$, hidden dimension for $1$-WL-GNN $h_1$: $\{32,~64,~96,~128\}$, number of hidden layers for $1$-WL-GNN $l_1$: $\{1,~2,~3\}$, number of hidden layers for $2$-WL-GNN $l_2$: $\{1,~2,~3\}$, hidden dimension for $2$-WL-GNN $h_2$ : $\{16,~24,~32,~64,~96\}$, dropout ratio for embedding layer $dp_1$, $1$-WL layer $dp_2$, $2$-WL layer $dp_3$: $\{0.1,~0.2,~0.3,~0.4,~0.5\}$. We use Optuna~\citep{akiba2019optuna} to perform random searching for hyperparameters.

\section{More experiments on knowledge graph datasets}\label{app:kgc}
In this section, we conduct an additional experiments to test $2$-WL-GNNs' link prediction performance on inductive knowledge graph completion (KGC). We adopt two datasets, FB15K-237 and WN18RR from \citep{teru2020inductive} to evaluate the performance. Each dataset includes four versions v1 to v4 with increasing sizes. The baselines we use are state-of-the-art inductive KGC methods including R-GCN~\citep{schlichtkrull2017modeling}, GraIL~\citep{teru2020inductive}, and a recent line-graph-based model INDIGO~\citep{liu2021indigo}. We compare them with our GNN implementations of $2$-WL$_L$ and $2$-FWL$_L$ using three metrics: accuracy (ACC), area under the ROC curve (AUROC) and Hits@3. The results are given in Table~\ref{tab:KG}. Best and second-to-best results are in bold and with underlines respectively.
\begin{table}[t]\small
    \centering
    \caption{Performance on KG datasets (\%). Higher the better.}
    \begin{tabular}{cc|cccc|cccc}
        \toprule
        &&\multicolumn{4}{c}{FB15K-237}&\multicolumn{4}{c}{WN18RR}\\
        && v1 & v2 & v3 & v4 & v1 & v2 & v3& v4\\
        \midrule
        \multirow{5}{*}{ACC}&R-GCN& 51.0& 51.3& 54.9& 52.1& 50.2& 52.7& 52.2&48.4\\
        ~ &GraIL& 69.0& 80.0& 81.0& 79.3& \textbf{88.7}& 81.2& 75.7&86.4\\
        ~ &INDIGO& 84.3& 89.3&89.0 &87.8 &\underline{85.7} &85.8 &\textbf{84.3} &85.4\\
        ~ &$2$-WL$_L$& \underline{85.7}& \underline{93.2} &\underline{90.0} &\underline{91.1} &84.7 &\underline{86.5} &\underline{79.9} & \underline{86.8}\\
        ~ &$2$-FWL$_L$& \textbf{90.7}& \textbf{94.7} &\textbf{93.9} &\textbf{91.8} &84.7 &\textbf{86.7} &81.5 & \textbf{88.7}\\
        \midrule
        \multirow{5}{*}{AUROC}&R-GCN& 51.0& 50.5& 50.5& 52.6& 49.0& 49.8& 53.1&50.2\\
        ~ &GraIL& 78.6& 90.0& 93.1& 89.5& \underline{92.3}& 92.7& 82.8& 94.4\\
        ~ &INDIGO& \underline{93.4}& \underline{96.3}& 96.6& 95.8& 91.2& 92.5& \textbf{92.4}& \underline{94.7}\\
        ~ &$2$-WL$_L$& 87.9& 95.7& \underline{96.9}& \textbf{97.7}& 88.5& \textbf{93.3}& \underline{86.6}& 89.1\\
        ~ &$2$-FWL$_L$& \textbf{95.3}& \textbf{98.2}& \textbf{97.5}& \underline{96.6}& \textbf{92.8}& \textbf{93.3}& 85.9& \textbf{95.4}\\
        \midrule
        \multirow{5}{*}{Hits@3}&R-GCN& 2.4& 3.4& 3.5& 3.3& 2.1& 11.0& 24.5& 8.1\\
        ~ &GraIL& 1.0& 0.4& 6.6& 3.0& 0.6& 10.7& 17.5& 22.6\\
        ~ &INDIGO& 53.1& 67.6& 66.5& 66.3& \textbf{98.4}& \textbf{97.3}& \textbf{91.9}& 96.1\\
        ~ &$2$-WL$_L$& \underline{70.8}& \underline{79.0}& \underline{79.5}& \underline{79.8}& \underline{97.8}& 96.1&83.7& \underline{96.2}\\
        ~ &$2$-FWL$_L$& \textbf{71.5}& \textbf{84.2}& \textbf{81.7}& \textbf{78.3}& 97.4& \underline{96.6}& \underline{85.2}& \textbf{97.3}\\
        \bottomrule
    \end{tabular}
    \label{tab:KG}
\end{table}

As we can see, $2$-FWL$_L$ generally achieves the strongest performance with \textbf{17 highest metric numbers out of 24}, and 3 second-to-best metric numbers among the remaining. This again verifies the higher link expressive power brought by the $2$-FWL tests. On the other hand, $2$-WL$_L$ and INDIGO perform competitively too. As discussed in the related work, INDIGO can be understood as a special implementation of local $2$-WL by leveraging line graphs. The excellent performance of $2$-WL methods further verifies the advantage of directly learning link representations. Specifically, we notice that these link-centered methods have much higher Hits@3 than the other node-centered baselines, indicating that link-centered methods are better at ranking the correct links at the top. This is especially important in real-world applications where we can only focus on top-ranked predictions.

\end{document}